\documentclass[accepted]{uai2024}

\usepackage[british]{babel}
\usepackage{natbib} 
\renewcommand{\bibsection}{\subsubsection*{References}}
\usepackage{mathtools} 

\usepackage{booktabs}       
\usepackage{amsfonts}       
\usepackage{siunitx}

\usepackage{amsthm}
\usepackage[ruled,linesnumbered]{algorithm2e}

\newcommand{\R}{\mathbb{R}}

\newcommand{\N}{\mathbb{N}}
\newcommand{\prob}{\mathbb{P}} 
\newcommand{\E}{\mathbb{E}} 

\newcommand{\D}{\mathcal{D}}

\newcommand{\sts}{\mathcal{S}} 
\newcommand{\as}{\mathcal{A}}  
\newcommand{\mdp}{\mathcal{M}}
\newcommand{\tmax}{t_{\max}}
\newcommand{\mdpfull}{\mdp=\left(\sts, \as, p, r, \gamma, \tmax, \rho_0\right)}

\newcommand{\ssucc}{\sts_{\text{succ}}}

\newcommand{\argmax}{\text{argmax} }

\newtheorem{theorem}{Theorem}

\title{Walking the Values in Bayesian Inverse Reinforcement Learning}

\author[1]{\href{mailto:<ondrej@bajgar.org>?Subject=Your UAI 2024 paper}{Ondrej Bajgar}{}}
\author[1]{Alessandro Abate}
\author[2]{Konstantinos Gatsis}
\author[1]{Michael A. Osborne}

\affil[1]{%
    University of Oxford
}
\affil[2]{University of Southampton}

\begin{document}

\maketitle

\begin{abstract}
    The goal of Bayesian inverse reinforcement learning (IRL) is recovering a posterior distribution over reward functions using a set of demonstrations from an expert optimizing for a reward unknown to the learner. The resulting posterior over rewards can then be used to synthesize an apprentice policy that performs well on the same or a similar task.
    A key challenge in Bayesian IRL is bridging the computational gap between the hypothesis space of possible rewards and the likelihood, often defined in terms of Q values: vanilla Bayesian IRL needs to solve the costly forward planning problem -- going from rewards to the Q values -- at every step of the algorithm, which may need to be done thousands of times. We propose to solve this by a simple change: instead of focusing on primarily sampling in the space of rewards, we can focus on primarily working in the space of Q-values, since the computation required to go from Q-values to reward is radically cheaper. Furthermore, this reversion of the computation makes it easy to compute the gradient allowing efficient sampling using Hamiltonian Monte Carlo. We propose ValueWalk -- a new Markov chain Monte Carlo method based on this insight -- and illustrate its advantages on several tasks.
\end{abstract}

\section{Introduction}
Reinforcement learning (RL) has shown impressive performance across a wide variety of tasks, ranging from robotics to game playing. However, one of the main challenges in applying RL to real-world problems is specifying an appropriate reward function by hand, which is often difficult and can result in reward functions that are only imperfect proxies for designers' intentions. Inverse reinforcement learning (IRL) addresses this issue by instead learning the underlying reward function from expert demonstrations.

A key challenge in IRL is that the reward function is often underdetermined by the available demonstrations, as multiple reward functions can lead to the same optimal behaviour. This can be solved by picking a criterion for choosing among the reward functions compatible with the demonstrations -- maximum margin~\citep{ng2000,ratliff2006} and maximum entropy \citep{ziebart2008} are the most prominent examples. As an alternative, Bayesian IRL explicitly tracks the uncertainty in the reward using a probability distribution. This not only accounts for the issue of underdeterminacy but also provides principled uncertainty estimates to any downstream tasks, which can be used, for instance, for the synthesis of safe policies or for active learning.

While having these attractive properties, Bayesian IRL is computationally challenging. While inference is done over the space of reward functions (in terms of which the prior is also expressed), the likelihood is usually formulated in terms of Q values (or is otherwise linked to the distribution of trajectories), and going from the former to the latter may require solving the whole forward planning problem at each iteration (as is case in the original Bayesian IRL algorithm \citep{ramachandran2007}), which is expensive in itself and may further need to be done thousands of times during IRL inference. To avoid this, we propose to use a simple insight: while going from rewards to Q-values is expensive, the inverse calculation can be much simpler. Thus, we propose to perform the inference as if it were done primarily over the space of Q-values, computing reward estimates beside it, resulting in a much cheaper algorithm. A related formulation appeared already in the variational method of \cite{chan2021}, which was, however, learning only a point estimate of the Q-function thus sacrificing Bayesianism from the centre of the algorithm. 

We instead propose a new method that provides a full Bayesian treatment of the Q values, along with the rewards, and is able to provide samples from the true posterior, being based on Markov chain Monte Carlo (MCMC) as opposed to variational inference, which needs to pre-specify a family of distributions within which to approximate the posterior. Furthermore, since the computation required at each step is much simpler than in prior MCMC-based methods \citep{ramachandran2007, michini2012}, which in itself makes our method more efficient, we can also easily calculate the gradient, which allows us to use Hamiltonian Monte Carlo \citep{duane1987} granting further gains in efficiency. 

The contributions of this paper are the following: (1) we provide the first MCMC-based (and thus agnostic to the shape for the posterior) algorithm for continuous-space Bayesian inverse reinforcement learning; (2) we show that it scales better on discrete-space cases than the MCMC-based baseline, PolicyWalk; and (3) we show that we outperform the previous state-of-the-art algorithm for Bayesian IRL on continuous state-spaces, AVRIL, better capturing the posterior over rewards and performing better on imitation learning tasks. 

The paper is organized as follows: Section 2 provides background on inverse reinforcement learning and Markov-chain Monte Carlo and summarizes related work. Section 3 introduces our proposed algorithm called ValueWalk. Section~\ref{sec:experiments} compares our approach to an MCMC-based predecessor, PolicyWalk \citep{ramachandran2007}, the previous state-of-the-art scalable method for Bayesian IRL, AVRIL \citep{chan2021}, and 2 imitation learning baselines on several control tasks. 

\section{Background}
\subsection{Bayesian inverse reinforcement learning}

The goal of Bayesian inverse reinforcement learning is recovering a posterior distribution over reward functions based on observing a set of demonstrations $\D=\{(\phi(s_1),a_1),...,(\phi(s_n),a_n)\}$ from an expert acting in a Markov decision process (MDP) $\mdpfull$ where 
 $\sts, \as$ are the state and action spaces respectively,
$\phi: \sts \to \Phi$ is a feature function representing states in a feature space $\Phi$,
$p: \sts \times \as \to \mathcal P(\sts)$ is the transition function where $\mathcal P(\sts)$ is a set of probability measures over $\sts$,
$r: \Phi\times \as \to  \mathbb{R}$ is a reward function,
$ \gamma\in (0,1) $ is a discount rate,
$\tmax\in\N\cup\{\infty\}$ is the time horizon, and
 $\rho_0\in\mathcal P(\sts)$ is the initial state distribution.

In IRL, we know all elements of the MDP except for the reward function and, possibly, the transition function (the setting without the knowledge of transition dynamics -- or other form of access to the environment or its simulator -- is sometimes called \emph{strictly batch} \citep{jarrett2020}; our method is applicable in both this setting and the one including an environment simulator, though most of the experiments are run in the former setting following the main baseline method, AVRIL).
Instead, we have a model of how the expert policy is linked to the reward and, in the case of Bayesian IRL, also a prior distribution over reward functions, $p_R$ (which is, in general, a multi-dimensional stochastic process, that for any set of state-action pairs returns a joint probability distribution over the corresponding set of real-valued rewards). Commonly used expert models include Boltzmann rationality models such as
\begin{equation}
\label{eq:boltzmann-rat}
\prob [a_i | \phi(s_t)] = \frac{e^{\alpha Q^*(\phi(s_t),a_i)}}{\sum_{a'\in\as} e^{\alpha Q^*(\phi(s_t),a')}}
\end{equation}
\citep{ramachandran2007,chan2021} where
$Q^*(s,a)$ is the expected (discounted) return if action $a$ is taken in state $s$, and the optimal policy is subsequently followed, and $\alpha$ is a rationality coefficient; the maximum entropy approach \citep{ziebart2008}, where the probability of each trajectory is assumed to be proportional to the exponential of the trajectory's return; or sparse behaviour noise models \citep{zheng2014}, where the expert is assumed to behave rationally except for sparse deviations. Beside these approximately rational models, various models of irrationality can also be considered \citep{evans2015}. The Bayesian IRL framework is flexible with respect to the choice of expert model, each such model just resulting in a different likelihood function, and can also be extended to the case where the model is not fully known. 

In this article, we adopt the Boltzmann rationality model (\ref{eq:boltzmann-rat}). We will assume that conditional on the Q values, the actions chosen by the expert are independent, yielding the likelihood 
\begin{equation}
\label{eq:boltzmann-likelihood}
p(\D|r) = \prod_{s_t,a_t,s_{t+1}\in\D} \frac{e^{\alpha Q^*(\phi(s_t),a_t)}}{\sum_{a'\in\as} e^{\alpha Q^*(\phi(s_t),a')}} p(s_{t+1}|s_t,a_t)
\end{equation}
for a discrete action space $\as$ (the expression can readily be adapted to a continuous setting by replacing the sum by an integral).
Given this likelihood together with the prior over rewards $p_R$, we can calculate the posterior using the Bayes Theorem as $p(r|\D) = p(\D|r)p_R(r)/p(\D)$. Generally, we cannot calculate this posterior analytically, so in practice, we need to resort to approximate methods. In this article, we use Markov chain Monte Carlo sampling.

When performing Bayesian inference over the reward, the transition probabilities will be considered fixed (except for Appendix~\ref{app:unknown-transitions}, which discusses the extension of Bayesian inference also to transition probabilities). Thus looking at the likelihood as a function of the reward, we can write 
\begin{equation}
    p(\D|r)=c \prod_{s_t,a_t\in\D} \frac{e^{\alpha Q^*(\phi(s_t),a_t)}}{\sum_{a'\in\as} e^{\alpha Q^*(\phi(s_t),a')}} =: c \mathcal{L}(\D|r).
\end{equation}
Since $p(D) = \int p(D|r) d p_R(r) = c\int \mathcal{L}(D|r) d p_R(r)$, the constant transition term cancels out in the posterior, and, going forward, we can use the partial likelihood $\mathcal{L}$ in reward posterior inference. Furthermore, MCMC algorithms generally depend only on the unnormalized distribution, thus we can also drop the remainder of the marginal $p(D)$ from our calculation. 

\subsection{Markov-chain Monte Carlo (MCMC)}
Markov chain Monte Carlo (MCMC) methods form a class of algorithms widely used for sampling from complex probability distributions. MCMC methods rely on constructing Markov chains whose stationary distribution is the distribution of interest. Usually a new candidate sample in the chain is proposed and then accepted or rejected with probability proportional to the one under the target distribution -- in our case the posterior over rewards.

In simpler MCMC methods, such as Metropolis-Hastings~\citep{metropolis1953,hastings1970}, which were also used in some previous articles on Bayesian IRL \citep{ramachandran2007,michini2012}, the new step is proposed as a random jump in the sampling space. However, this often leads to a high rejection rate, if the jumps are large, or tightly correlated samples, if the jump is small, both of which can make the algorithm inefficient. 

Thus, we instead use the popular Hamiltonian (or hybrid) Monte Carlo (HMC; \cite{duane1987}) with the no-U-turn (NUTS) sampler \citep{hoffman2014}, which uses the gradient of the posterior density and Hamiltonian-like dynamics to propose samples that are far apart but still likely under the posterior, keeping a high acceptance rate, thus improving the efficiency of the algorithm.

\subsection{Related Work}

Inverse reinforcement learning is most often used as a component in imitation learning: the more general task of learning an apprentice policy from expert demonstrations (see \cite{zare2023} for a good recent survey). Beside IRL, the other major family of methods within imitation learning is behavioural cloning \citep{pomerleau1991,ross2011}, which, in its vanilla form, aims to learn the policy via supervised learning directly from the expert's observation-action pairs. The supervised learning approach has an advantage of lower computational cost, but faces the challenge of covariate-shift, since the training states are distributed according to the expert policy, not that of the learner agent, though multiple methods try to mitigate this by encouraging the learner policy to stay close to the expert one \citep{dadashi2020,reddy2019,brantley2019}.

Inverse reinforcement learning represents an alternative which, instead of directly learning the observation-action mapping, first learns an estimate of the reward function, which can then be used to synthesize a policy. This can offer better generalization, but usually requires a model of the environment or access to it in order to run reinforcement learning, and generally incurs a higher computational cost.

We build on the paradigm of Bayesian IRL introduced by \cite{ramachandran2007}. While the Bayesian approach is attractive thanks to its principled treatment of uncertainty in light of the limited demonstration data, the key downside relative to other methods has been its scalability to higher-dimensional settings. \cite{michini2012} try to improve efficiency upon Ramachandran by focusing computation into regions of the state space close to the expert demonstration, still using MCMC, while \cite{chan2021} try to improve efficiency by using an approximate variational distribution to model the posterior, as well as an additional neural network that tracks the Q function, which avoids the need for a costly inner-loop solver. \cite{mandyam2023} has recently used kernel density estimation as an alternative method for approximate Bayesian inference.\footnote{The evaluation in this paper focuses on an offline setting without access to environment dynamics, while the last mentioned method fundamentally depends on having access to the environment dynamics so we omit it from the comparison in this paper.}

As opposed to recent work experimenting with other approximation techniques, we return to MCMC, with its greater expressivity, while at the same time adapting it to be used with continuous state spaces, which would not be feasible with prior MCMC-based methods.

\section{Method}

Similarly to early work in Bayesian IRL \citep{ramachandran2007, michini2012}, we use Markov chain Monte Carlo sampling to produce samples from the posterior distribution over rewards given a prior and expert demonstrations. Our key innovation is in the way we calculate the posterior. At each step of the Markov chain, these previous methods generally (1) proposed a new reward (2) used some method of forward planning, such as policy iteration, to deduce the corresponding optimal Q function and then (3) used the Q function to evaluate the likelihood and the reward to evaluate the prior.

We suggest proceeding the other way round: our method proposes a set of new parameters of the Q function and then uses it to deduce the corresponding rewards, which is generally a much easier calculation than going from rewards to Q functions. The method then uses the reward to calculate the prior and the Q value to evaluate the likelihood, and combines the two to calculate the unnormalized posterior density. This value can then be used for calculating the acceptance probability in any chosen MCMC algorithm. Also thanks to the calculation being simple (rather than involving a RL-like inner-loop problem) and differentiable, we can also calculate the gradient, which we can use for efficient proposals using HMC+NUTS. Since we construct the random chain in the space of Q values instead of the space of rewards, used by previous methods, we call our new method ValueWalk.

\subsection{Finite state and action spaces}
\label{ssec:method-finite}
Let us first outline the algorithm for the case of finite state and action spaces since the calculation can be performed exactly in this case, and the later continuous algorithm builds on this base case. We concentrate here on the calculation of the posterior probability corresponding to a single proposed set of Q values (which is performed at each step of the HMC trajectory) and otherwise employ standard HMC. Note that here, we assume the knowledge of the environment dynamics $P$, since this finite setting is close to that of PolicyWalk~\citep{ramachandran2007}, which also assumes this knowledge. However, the method can easily be extended to the \textit{strictly batch} setting using steps analogous to the ones taken in the next subsection on continuous spaces, or can be combined with inference over transition probabilities (see Appendix~\ref{app:unknown-transitions}).

In this finite case, we perform inference over a vector $Q\in\R^{|\sts||\as|}$ representing the optimal Q-value for each state-action pair. The first thing to notice is that given such a vector, we can calculate the corresponding reward vector of the same dimensionality as $Q$ using the Bellman equation as
\begin{equation}
\label{eq:reward-bellman-finite}
R(s,a) = Q(s,a) - \gamma \sum_{s'\in\sts} p(s'|s,a) \sum_{a'\in\as} \pi_Q(a'|s') Q(s',a')    
\end{equation}
with either $\pi^Q(a'|s')=\mathbb{I}[a'=\argmax_{a''}Q(s',a'')]$ or a softmax approximation (which we use since it has the advantage of being differentiable using an inverse temperature coefficient $\bar \alpha$ to regulate the softness of the approximation).  Equation (\ref{eq:reward-bellman-finite}) can also be written in vector form as
$R = (I - \gamma \bar P) Q$ where $\bar P$ is a $|\sts||\as|\times|\sts||\as|$ matrix whose values are defined as $\bar P(s,a;s',a')=P(s'|s,a)\pi^Q(a'|s')$. In that case, given a prior $p_R$ over rewards, we can calculate the prior of $Q$ as
$$p_Q(Q) = p_R((I - \gamma \bar P)Q) \det(I-\gamma \bar P),$$
where $p_Q$ and $p_R$ are the prior probability densities of $Q$ and $R$ respectively. Since $\bar P$ is a stochastic matrix and $0<\gamma<1$, the determinant is always strictly positive. Note that the determinant needs to be recalculated only if the optimal policy changes and otherwise can be cached between steps of HMC eliminating the associated costly calculation. Furthermore, we found that in practice, the recovered samples do not differ significantly if the determinant term is omitted.

The above prior term can be combined with the likelihood
$$\mathcal{L}(D|Q) = \prod_{(s,a)\in\D} \exp(\alpha Q(s,a))/\sum_{a'\in \as}\exp(\alpha Q(s,a'))$$
to calculate the unnormalized posterior density $p(Q|\D) \propto p_Q(Q) \mathcal{L}(\D|Q)$ which we use in the standard HMC+NUTS algorithm to produce samples from the posterior over Q values, and, as a byproduct, also samples from the posterior over rewards (as would be expected from an IRL algorithm). Algorithm~\ref{alg:posterior-finite} summarizes the whole posterior probability calculation, and Theorem~\ref{thm:detailed-balance} in Appendix~\ref{app:proofs} formally shows that the secondary Markov chain over rewards produced by the algorithm also satisfies the detailed balance condition with respect to the posterior over rewards and thus constitutes a valid MCMC algorithm for sampling from the reward posterior.

\begin{algorithm}[t]
\KwData{a candidate vector of Q values, set of expert demonstrations $\D$, prior over rewards $p_R$}
\For{$s,s'\in \sts, a,a'\in \as$}{
    $\pi^Q(a'|s')=\mathbb{I}[a'=\argmax_{a''}Q(s',a'')]$ \;
    $\bar P(s,a;s',a') = p(s'|s,a)\pi(a'|s')$ \; 
    }
$ R = (I - \gamma \bar P) Q$ where $\bar R,\bar Q$ \;
$p_Q(Q) = p_R( R) \det(I-\gamma \bar P)$ \;
$\mathcal{L}(\D|Q) = \prod_{(s,a)\in\D} \exp(\alpha Q(s,a))/\sum_{a'\in \as}\exp(\alpha Q(s,a'))$ \;
\KwResult{$p(Q|\D) \propto p_Q(Q) \mathcal{L}(\D|Q)$; candidate sample $ R$}
\caption{Calculation  of the unnormalized posterior for finite $\sts$ and $\as$ and known transition probabilities $P$ (performed in each step of HMC). The resulting candidate reward sample $\bar R$ is then accepted/rejected together with the corresponding Q.}
\label{alg:posterior-finite}
\end{algorithm}

See Section~\ref{ss:exp-gridworld} for an example of this finite-case algorithm applied to a gridworld environment. Note that if the reward is known to depend only on the state, the sampling can instead be performed over state-values $V$. Similarly, if it depends on the full state, action, next state triple, it should be performed over state-action-state values to maintain a match in the dimensionality of the reward and value spaces.

\subsection{Continuous state representations}
\label{ssec:method-cts}
For continuous or large discrete spaces, it is generally no longer possible or practical to maintain a separate Q-function parameter for each state, so we need to resort to approximation. Thus, from now on, our inference will centre around parameters $\theta_Q\in\R^{n_Q}$ of a Q function approximator $Q_\theta:\Phi\times\as\to\R$ where $\Phi$ is the space of feature representations of the states. While the method is again centred around the Q function, the algorithm can also produce samples from the \emph{reward} posterior at any set of evaluation points of interest, $\D_{\text{eval}}$. Furthermore, a method such as warped Gaussian processes \citep{snelson2003} can then be used to generalize the reward posterior from $\D_{\text{eval}}$ to new parts of the state-action space.

The likelihood calculation remains very similar to the discrete case:
\begin{equation}
\label{eq:boltzmann-rat-cts}
\mathcal{L}(\D|\theta_Q) = \prod_{(s,a)\in\D} \frac{\exp(\alpha Q_{\theta_Q}(\phi(s),a))}{\sum_{a'\in \as}\exp(\alpha Q_{\theta_Q}(\phi(s),a'))}
\end{equation} (assuming $\as$ to be bounded).
What concerns the evaluation of the prior, the reward corresponding to given Q-function parameters can be expressed using the continuous Bellman equation as
\begin{equation*}
    R(s,a) = Q_{\theta_Q}\bigl(\phi(s),a\bigr) - \gamma\E_{s',a'|s,a}\Bigl[Q_{\theta_Q}\bigl(\phi(s'),a'\bigr)\Bigr]
\end{equation*}
on any subset of states and actions. 

In general, the integral in $\E_{s',a'|s,a}[Q_{\theta_Q}(\phi(s',a')]=\int_{s'\in\sts} p(s'|s,a) \max_{a'\in\as}Q_{\theta_Q}(\phi(s'),a')$ needs to be approximated, for which any of a number of numerical methods can be used, from grid sampling to Monte Carlo methods, to more sophisticated techniques like probabilistic numerics \citep{hennig2022}. For most such methods, we the integral is approximated using a discrete set of candidate successor states $S_{\text{succ}}(s,a)=\bigl\{s\sim q(\cdot|s,a) \bigr\}$ sampled from some proposal distribution $q$ as
\begin{equation}
\label{eq:reward-importance-sampling}
\frac{1}{|\ssucc|}\sum_{s'\in \ssucc} \frac{p(s'|s,a)}{q(s'|s,a)} \max_{a'\in\as}Q_{\theta_Q}(\phi(s'),a').
\end{equation}

The variant of the approximation we choose depends of what information we have at our disposal:
\begin{itemize}
    \item If we have access to a probabilistic model $\hat p$ of the environment (which can either represent the true environment dynamics, if we know them, or our best inferred model of the dynamics including any epistemic uncertainty) that we can sample from, we can simply sample $\ssucc(s,a)=\{ s'\sim \hat p (\cdot|s,a) \}$ and drop the importance weight.
    \item If we can evaluate the density $\hat p$ we can directly use the importance sampling equation \ref{eq:reward-importance-sampling} with $q$ being a proposal distribution ideally close to $\hat p$.
    \item If all we have is a static set of trajectories $\D_+$ -- either just the expert ones $\D$, or also additional ones sampled from another, possibly random, policy -- we can crudely approximate the reward for a transition $s,a,s'\in\D_+$ using a singleton $\ssucc(s,a)=\{s'\}$. This is an approximation made by the baseline AVRIL algorithm, so to match, we use it for the experiments in Section~\ref{ssec:exp-classic-control}. In that case we require that $\D_{\text{eval}} \subseteq \D_+$, and for $s,a,s'\in\D_+$ we can define an empirical transition model $\hat p(s''|s,a)=\delta_{s'}(s'')$ to be used within the algorithm.
\end{itemize}

\begin{figure*}
    \centering
    \raisebox{0.025\height}{\includegraphics[width=0.30\textwidth]{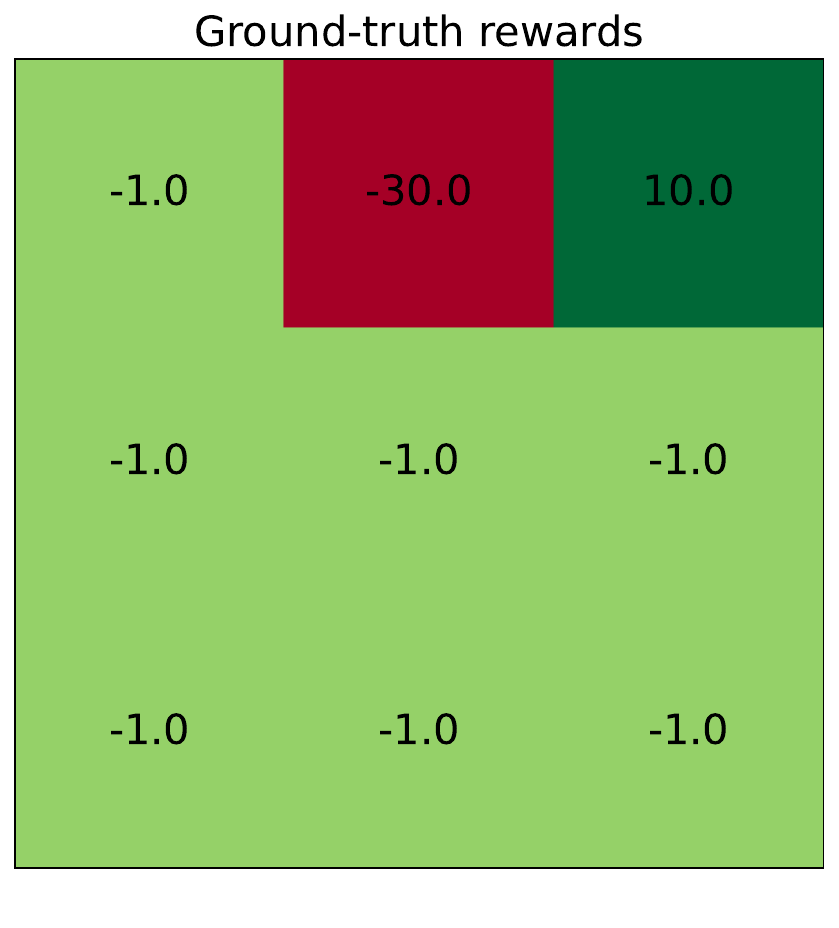}}\includegraphics[width=0.34\textwidth]{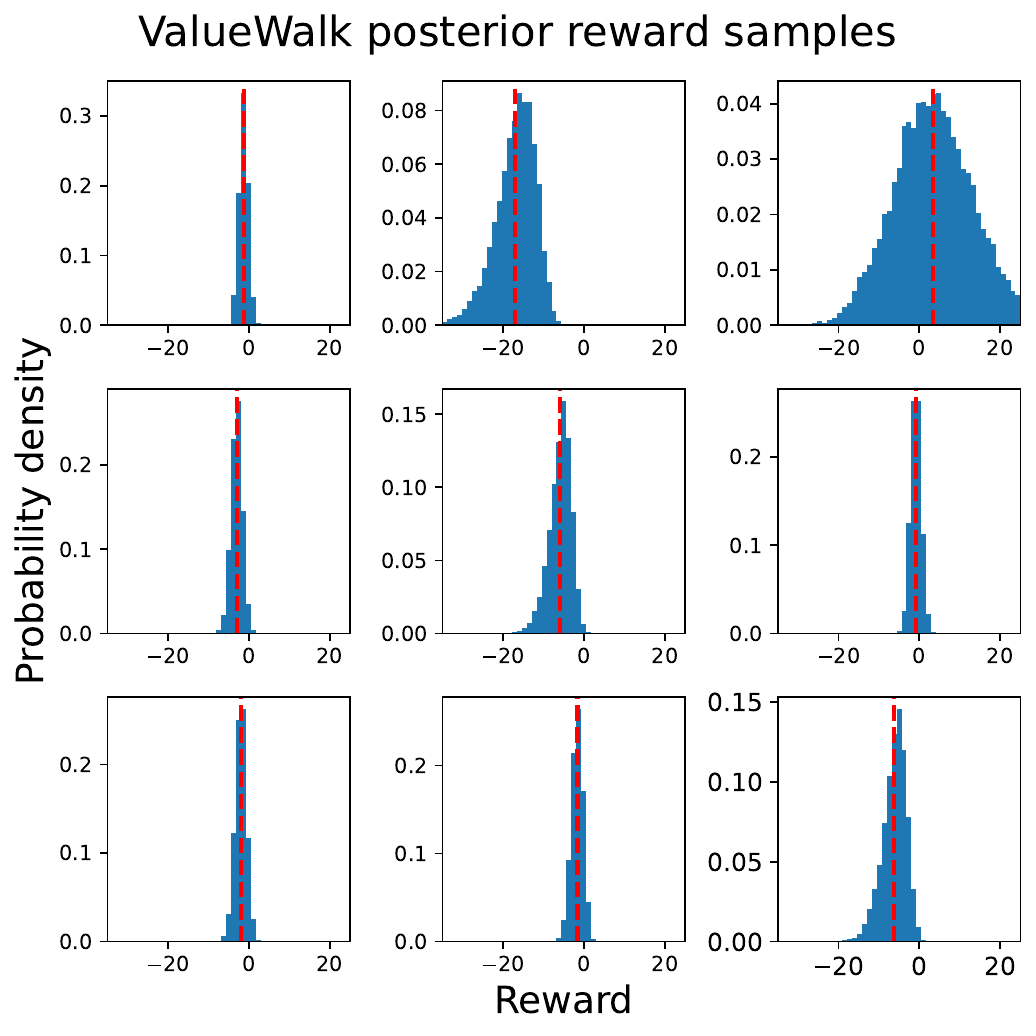}
    \includegraphics[width=0.34\textwidth]{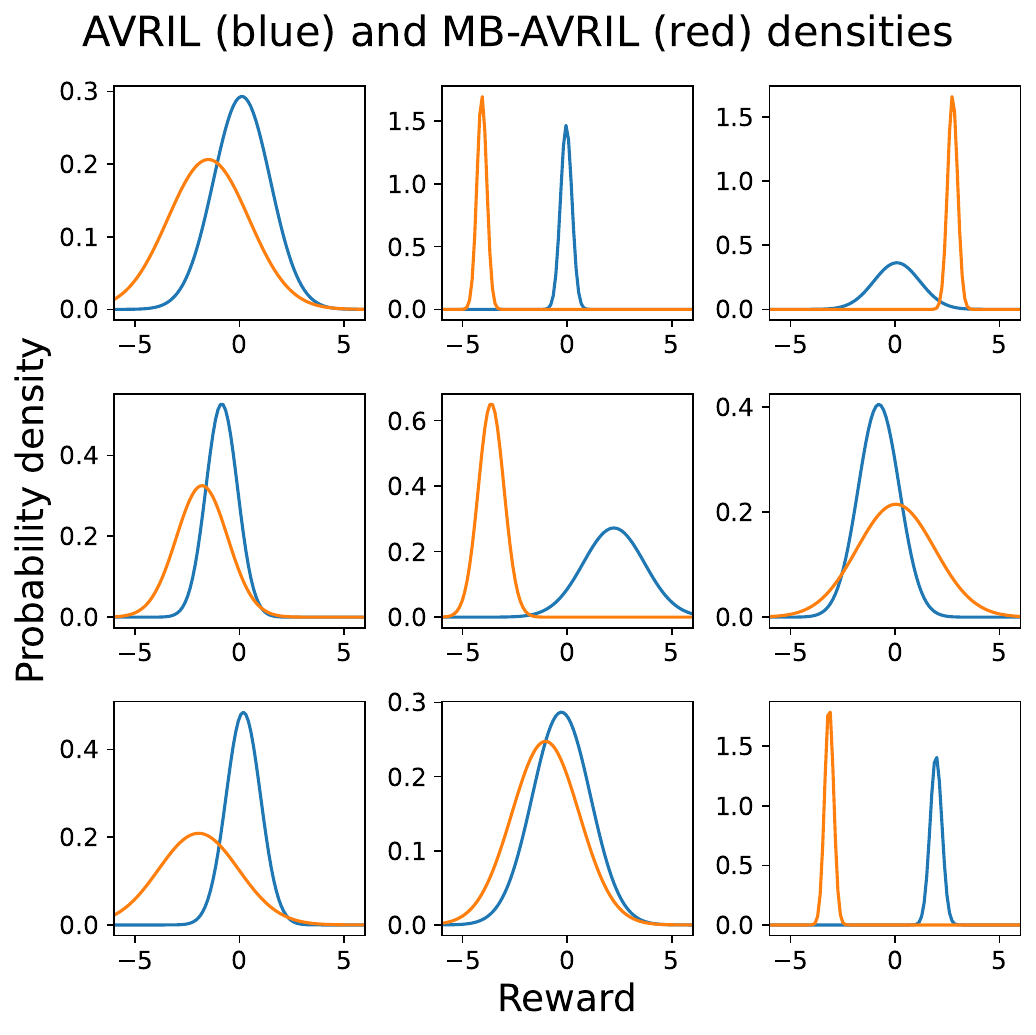}
    \caption{\textbf{Left}: Illustrative 3x3 gridworld. The agent always starts in the top left corner. The top right corner yields a reward of 10 and is terminal. The top centre tile represents an unsafe state that should be avoided and yields a reward of -30. \textbf{Centre}: Histograms of the samples from the posterior over rewards recovered by our ValueWalk algorithm corresponding to the 9 states of the gridworld. The red line indicates the mean. \textbf{Right}: Density functions of the posterior over rewards recovered by AVRIL and its model-based version, MB-AVRIL. Note the much narrower range of the reward axis relative to the histograms.}
    \label{fig:3x3_gw}
\end{figure*}

The corresponding continuous version of the algorithm is presented in Algorithm~\ref{alg:posterior-sim}. 

\begin{algorithm}[t]
    \KwData{candidate parameters of the Q-function $\theta_Q$, a set of expert demonstarations $\D$, a set of evaluation locations $D_{\text{eval}}$, prior over rewards $p_R$}
    Initialize empty sequence $\mathcal{R}_{\text{cand}}$ of candidate reward samples \;
    \For{$(s,a)\in\D_{\text{eval}}$}{
            Sample a set of successor states $\ssucc=\{s''\sim\hat p(\cdot|s,a)\}$\;
            $ R(s,a) = Q_{\theta_Q}(\phi(s),a) - \gamma \frac{1}{|\ssucc|} \sum_{s'\in\ssucc}\max_{a'\in \as} Q_{\theta_Q}(s',a')$\;
        
        Append $R_t$ to $\mathcal{R}_{\text{cand}}$\;
    }
   Use samples to evaluate the prior $p_R(D_{\text{eval}}, \mathcal{R}_{\text{cand}})$ \;
   Use demonstrations to evaluate the likelihood $\mathcal{L}(\D|\theta_Q)$ per equation (\ref{eq:boltzmann-rat-cts}) \;
    \KwResult{unnormalized approximate posterior $p(\theta_Q|\D)\propto p_R(D_{\text{eval}}, \mathcal{R}_{\text{cand}}) p(\D|\theta_Q)$; candidate reward samples $\mathcal{R}_{\text{cand}}$.}
\caption{Calculation of the unnormalized posterior probability with continuous state representations for a single proposed parameter value $\theta_Q$ (performed in each step of MCMC). The returned candidate reward samples are accepted or rejected by the outer MCMC algorithm together with the candidate parameters $\theta_Q$.}
\label{alg:posterior-sim}
\end{algorithm}

We can store both the Q function parameters $\theta_Q$ and the corresponding reward samples depending on downstream needs. We can then fit a warped Gaussian process to the posterior reward samples to get a posterior reward distribution over the whole state space. This can then be used together with an algorithm for RL (or \emph{safe} RL in particular) to find an apprentice policy from the reward. Alternatively, as a shortcut, the posterior over Q-functions can be used to define an apprentice policy directly.

\subsection{Continuous actions}
The algorithm can be extended to continuous actions, replacing the sum in the Boltzmann likelihood (\ref{eq:boltzmann-rat-cts}) by an integral, and again, in turn, approximating it by a discrete set of samples from the action space. Simple discretizations (such as uniform sampling) can work well for low-dimensional action spaces (as we illustrate in our safe navigation experiment in the next section) but suffer from the curse of dimensionality, so a more sophisticated scheme would be needed for higher-dimensional action spaces. We leave that for future work.

\section{Experiments}
\label{sec:experiments}

We tested our method on gridworlds (for illustration and to compare the speed to PolicyWalk~\citep{ramachandran2007}, which our method builds upon but which is restricted to such small finite-space settings) and on 3 simulated control tasks with continuous states.

\subsection{Gridworld}
\label{ss:exp-gridworld}

For an illustration of the method with easily interpretable and visualizable features, we first test it on a simple gridworld environment shown in Figure~\ref{fig:3x3_gw}. We have generated a fixed set of 50 demonstration steps in the environment and used our method, ValueWalk (including the environment dynamics), the original PolicyWalk~\citep{ramachandran2007}, a sped-up version of PolicyWalk using HMC, which we denote by PolicyWalk-HMC (see Appendix~\ref{app:pw-hmc}), and AVRIL~\citep{chan2021} (which does not use environment dynamics, making the comparison unfair but illustrative of inherent limitations of such model-free methods) as well as a model-based version of AVRIL, which we mark as MB-AVRIL (see Appendix~\ref{app:mb-avril}), to recover a posterior over rewards from an independent normal prior with mean 0 and standard deviation of 10. 

With ValueWalk and PolicyWalk-HMC, we collected a total of 1,000 MCMC samples using HMC+NUTS with 100 warm-up steps, which lead to $\hat R \leq 1.01$ on each dimension (where $\hat R$ is the potential scale reduction factor \citep{gelman1992}, a commonly used indicator that the chains have mixed well). For vanilla PolicyWalk, we collected 1M samples (since those are much more correlated). We then also ran PolicyWalk, PolicyWalk-HMC and ValueWalk on a 6x6 and 12x12 version of the gridworld to examine how the compute times of these MCMC-based methods scale.

\begin{table}[h]
    \caption{\textbf{Speed comparison.} Time per effective sample (in seconds) produced by PolicyWalk, PolicyWalk-HMC, and ValuWalk on a 3x3, 6x6, and 12x12 gridworld respectively.}
    \centering
    \begin{tabular}{r S[table-format=3.2] S[table-format=2.2] S[table-format=1.2]}
        \toprule
        {States} & {PolicyWalk} & {PolicyWalk-HMC} & {ValueWalk} \\
        \midrule
        9 & 0.86 & 0.80 & \textbf{0.20} \\
        36 & 9.00 & 4.18 & \textbf{0.71} \\
        144 & 246.43 & 18.44 & \textbf{0.77} \\
        \bottomrule
    \end{tabular}
    \label{tab:time-on-gridworlds}
\end{table}

\subsubsection{Results}
Both PolicyWalk and ValueWalk (our algorithm) resulted in matching posterior reward samples as expected (Kolmogorov-Smirnov did not reveal any significant differences with all p-values > 0.2 on each of the 9 dimensions of the reward). The histograms of the samples are shown in the middle plot in Figure~\ref{fig:3x3_gw}. The speed comparison of the two methods can be found in Table~\ref{tab:time-on-gridworlds}, showing ValueWalk indeed runs faster than the baseline PolicyWalk algorithm in both variants, with the advantage growing with an increasing size of the environment and a correspondingly growing number of reward parameters.

The posterior tends to concentrate around the ground truth value, except in the terminal top-right state, which shows that the data are consistent with both positive values and mildly negative ones (since it is a terminal state, the fact that the expert heads there can equally well be explained by escaping from negative-reward states as by trying to incur a positive reward. This is confirmed if we look at the correlation between various rewards shown in Figure~\ref{fig:3x3_2d_hists} in Appendix~\ref{app:gw-details} which shows that the reward in the terminal state can be negative only if other states' rewards are also negative).

We also ran AVRIL on this simple gridworld (which took 43s to converege). In terms of the resulting posterior, there are 3 things to note (see Figure~\ref{fig:3x3_gw} centre and right). Firstly, the posteriors of some states are much tighter -- the x-axis is zoomed in about 4x relative to the ValueWalk histograms. This is due to the fact that AVRIL does not model the uncertainty in the Q-function, instead learning only a point estimate. The reward posterior is then pegged to this Q-function point estimate thus significantly reducing its variance. As a result, both the reward of the obstacle and of the goal are extremely unlikely under the posterior.

Secondly, we can observe that the posterior reward for the obstacle is not any lower than that for most other states. This is because this state is never visited in demonstrations, and AVRIL -- not taking the environment dynamics into account -- consequently does not update this value. This illustrates an important downsides faced by methods without an environment model. (Note that the model-free version of ValueWalk would face the same issue.) This defect is fixed in the model-based version of AVRIL, AVRIL-MB.

Finally, we can see that while the true posterior differs from normal (see especially the strong skew of the negative-reward top middle cell), AVRIL is limited by its normal variational distribution. While in theory, AVRIL could be used with any variational family, we first need to determine which family may be suitable, for which an MCMC-based method such as ours is a useful instrument.

\begin{figure*}
    \centering
    \includegraphics[width=0.32\textwidth]{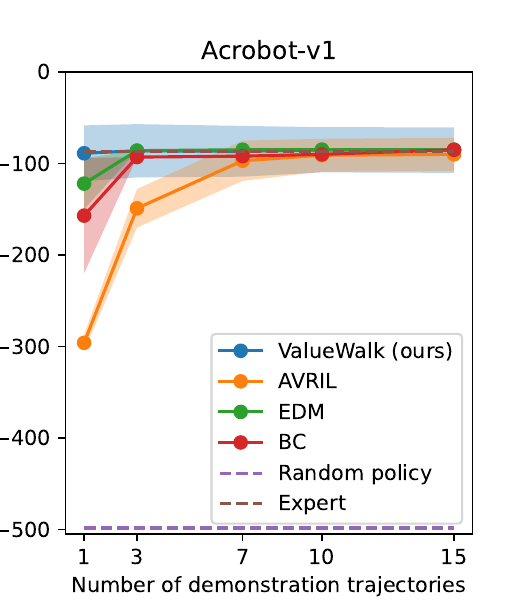}
    \includegraphics[width=0.32\textwidth]{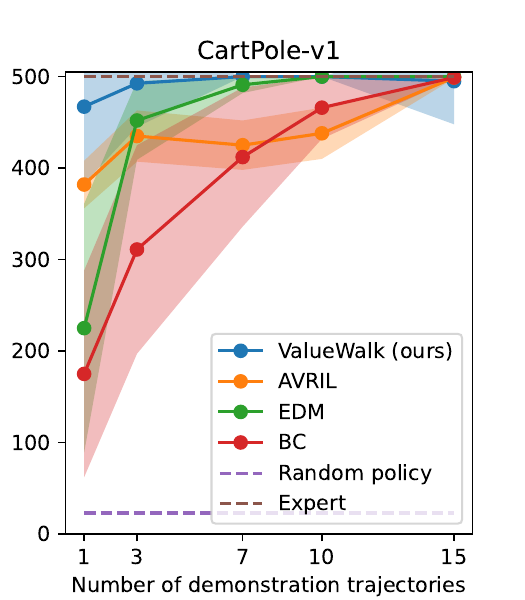}
    \includegraphics[width=0.32\textwidth]{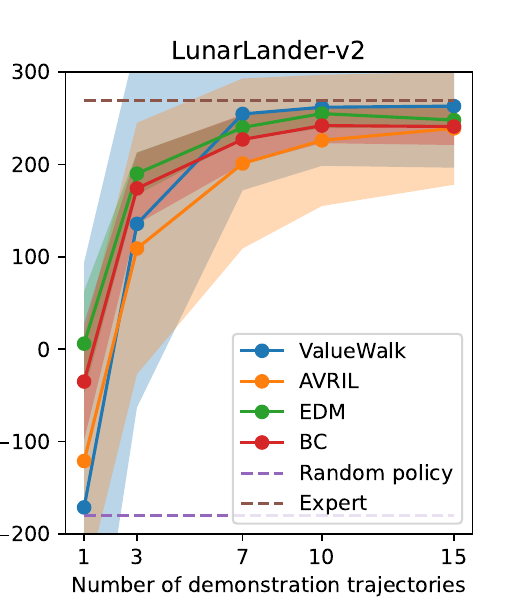}

    \caption{The test performance of an apprentice agent for ValueWalk and 3 baseline methods for different numbers of demonstration trajectories. The ValueWalk apprentice agent takes the action that maximizes the median of the posterior Q-value samples. The line shows mean performance across 5 runs with different sets of expert demonstrations; the shaded region shows mean$\pm$std.}
    \label{fig:classic_env_results}
\end{figure*}


\subsection{Classic Control Environments}
\label{ssec:exp-classic-control}

To allow for direct comparison, we also evaluated ValueWalk on three classic control environments that were used to evaluate AVRIL by its authors: CartPole, where the goal is to balance an inverted pendulum by controling a cart underneath it, Acrobot, where the goal is to swing up a double pendulum using an actuated joint, and LunarLander, where the goal is to safely land a simulated lander on the surface of the moon. We used the same setup as was used for AVRIL to study the performance of an apprentice agent as a function of the number of demonstration trajectories for 1, 3, 7, 10, and 15 trajectories. The apprentice agent was evaluated on 300 test episodes and the mean reward is reported. We also compare against energy-based distribution matching (EDM; \cite{jarrett2020}) -- a successful method for strictly batch imitation learning -- and plain behavioural cloning (BC) as a simple baseline. Baseline results were taken from \cite{chan2021}.

\subsubsection{Results}

\begin{figure*}
    \centering
    \includegraphics[width=0.32\textwidth]{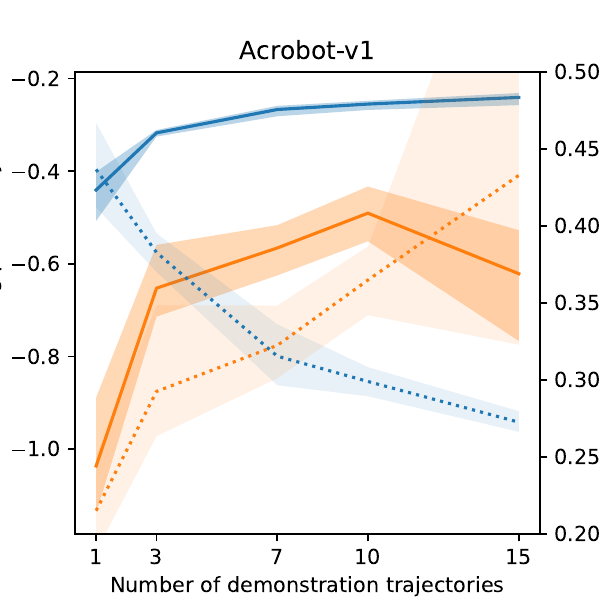}
    \includegraphics[width=0.32\textwidth]{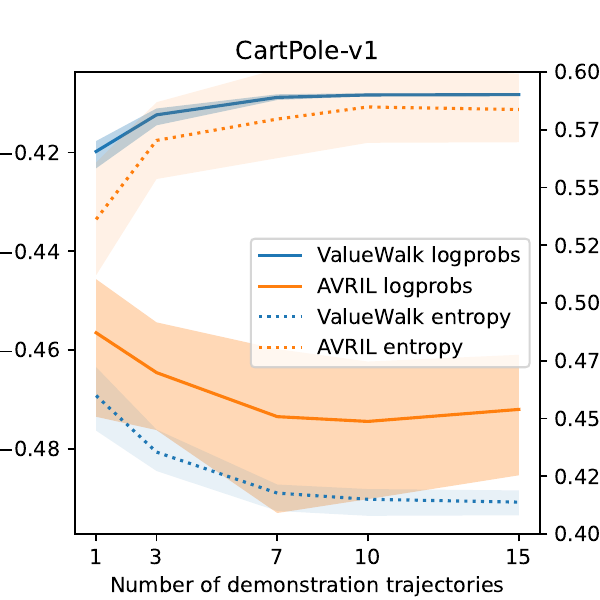}
    \includegraphics[width=0.32\textwidth]{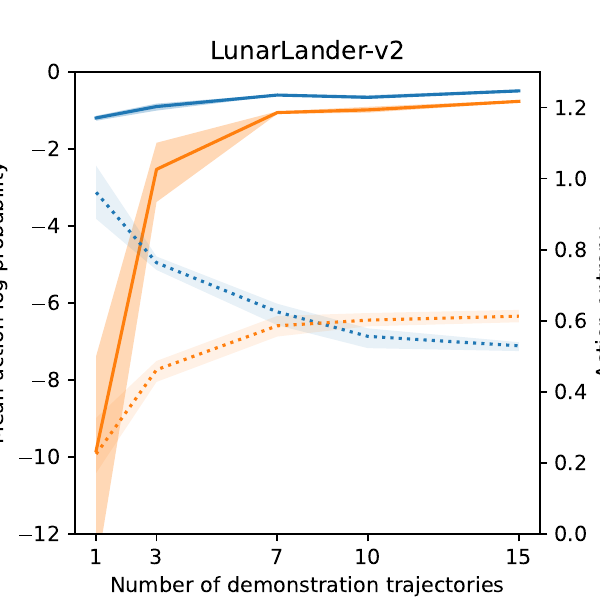}

    \caption{The log likelihood on a hold-out set of 100 test demonstrations and the entropy of the action predictions produced by ValueWalk and AVRIL. The plot shows the mean and the 90\% confidence interval on the value of the mean calculated using the bootstrap.}
    \label{fig:classic_env_logprobs}
\end{figure*}

The results are plotted in Figure~\ref{fig:classic_env_results}. While both agents do close to expert-level when provided with 15 expert trajectories, our algorithm reaches this level with fewer expert demonstrations. We hypothesize that this is due to treating the Q-function in a Bayesian way, as opposed to a point estimate in AVRIL, leveraging the advantages of a fully Bayesian treatment in the low data regime. 

To support this, we can look at the log likelihoods of the action predictions on a hold-out set of 100 test trajectories and the entropies of the predictive posterior shown in Figure~\ref{fig:classic_env_logprobs}. For ValueWalk, the log likelihood increases as the method is given more trajectories, while the prediction entropy decreases as we would expect from a Bayesian method given increasing amounts of information. On the other hand, we do not consistently see similar behaviour in AVRIL. The test log likelihood consistently increases only in the case of the LunarLander environment, where it, however, starts from extremely low levels (the initial \emph{mean} log probability of -10 would correspond to a probability of $5*10^{-5}$, suggesting the method has been putting practically 0 probability on actions taken by the expert among only 4 possible actions). Also, the prediction entropy of AVRIL tends to increase with seeing more trajectories. That suggests that AVRIL may be exhibiting overfitting behaviour in the low data regimes, which Bayesian methods should generally avoid.

The ValueWalk experiments on the control environments were run for 10,000 sample steps with 4,000 warm-up steps on Lunar Lander and 5,000 sample steps with 2,000 warm-up steps on Cartpole and Acrobot. The training takes between 2 and 19 hours of wall time on a single Nvidia RTX 3090 GPU\footnote{Experiments with fewer trajectories were run on a CPU.} where AVRIL takes 1-5 minutes to converge.

\section{Discussion}
We presented a method that allows us to apply MCMC-based Bayesian inverse reinforcement learning to continuous environments. The method maintains the attractive properties of MCMC methods: it is agnostic to the shape of the posterior (where variational methods assume a particular parameterized distribution family) and given enough compute, produces samples from the true posterior. This comes at a large computational cost relative to cheaper methods, such as variational inference. However, we still think MCMC-methods do have a role to play in the Bayesian IRL ecosystem. 

Firstly, we have shown that staying true to the Bayesian posterior does bring benefits in terms of superior performance on imitation learning tasks. Furthermore, the computational cost is paid in the learning phase, with inference at deployment being fast (sub millisecond per step in all cases, which would be sufficient for real-time control in most possible use cases and could be further optimized).

Secondly, we think that having a method that can draw samples from the true posterior can be extremely important in the process of developing other, faster or easier to scale methods, since it allows us to assess how their approximation deviates from the true posterior and how it impacts their performance. Also, variational methods in particular require a pre-specified family of distributions over which the optimization is subsequently run. ValueWalk can be used in an exploratory phase to determine what family of distributions may be appropriate for the problem at hand, before possibly using the advantages of variational methods to scale up. 

Thus, despite their steep computational cost, we think MCMC methods have their place in Bayesian inverse reinforcement learning, and our method is a sizable step in extending them up to a wider range of settings.

\bibliography{2305_zotero}

\newpage

\onecolumn

\title{Walking the Values in Bayesian Inverse Reinforcement Learning\\(Supplementary Material)}
\maketitle

\appendix




\section{Unknown transition probabilities}
\label{app:unknown-transitions}
Section~\ref{ssec:method-finite} presents a version of the ValueWalk algorithm for finite state and action spaces that assumes known transition probabilities. However, the key trick used in ValueWalk extends to unknown transition probabilities as well. 

One simplified option to handle unknown transitions, also employed in the continuous-state case in Section~\ref{ssec:method-cts} matching the setting used by AVRIL, is replacing the transition probabilities with their empirical estimate $\hat p(s'|s,a) = \xi(s,a,s')/\xi(s,a)$ where $\xi(s,a,s'),\xi(s,a)$ are the numbers of occurrences in the set of demonstration set of the transition $(s,a,s')$ and state-action pair $(s,a)$. In the finite-state, this would mean limiting the evaluation of the prior in Algorithm~\ref{alg:posterior-unknown-transitions} to only those state-action pairs that do occur in the data (i.e. replacing vectors and matrices on lines 3-6 by the appropriate sub-vectors and sub-matrices).

A more principled Bayesian alternative is of course using full Bayesian inference also over transitions -- in that case, we can perform the MCMC sampling jointly over both the transitions and the Q function parameters, recovering samples from the full joint posterior. The changes needed are (1) treating parameters of the transition model as inputs in the algorithm, (2) adding a prior over those parameters (so the joint prior will be a product of the Q-parameter prior and the transition-parameter prior), and (3) including transition probabilities in the likelihood. Here is the adaptation of the finite-space algorithm to this case of unknown probabilities:
\begin{algorithm}
\KwData{a candidate matrix of Q values, a candidate transition matrix $P$, set of expert demonstrations $\D$, prior over rewards $p_R$, prior over transitions $p_P$}
\For{$s,s'\in \sts, a,a'\in \as$}{
    $\pi(a|s) = \exp(\bar\alpha Q(s,a))/\sum_{a'\in \as}\exp(\bar\alpha Q(s,a'))$\;
    $\bar P(s,a;s',a') = P(s'|s,a)\pi(a'|s')$ \; 
    }
$\bar R = (I - \gamma \bar P) \bar Q$ where $\bar R,\bar Q$ are flattened vector versions of the reward and Q-value matrices \;
$p_Q(Q) = p_R(\bar R) \det(I-\gamma \bar P)$ \;
$p(\D|Q) = \prod_{(s,a,s')\in\D} P(s'|s,a)\exp(\alpha Q(s,a))/\sum_{a'\in \as}\exp(\alpha Q(s,a'))$ \;
\KwResult{$p(Q,P|\D) \propto p_P(P) p_Q(Q) p(\D|Q,P)$; candidate sample $\bar R$}
\caption{Calculation  of the unnormalized posterior for finite $\sts$ and $\as$ with unknown transition probabilities (performed in each step of HMC). The resulting candidate reward sample $\bar R$ is then accepted/rejected together with the corresponding Q and P.}
\label{alg:posterior-unknown-transitions}
\end{algorithm}

\section{Proof of soundness of the algorithm}
\label{app:proofs}

\begin{theorem}
\label{thm:detailed-balance}
Assume that the transition kernel $q_Q$ satisfies the detailed balance condition
$$\frac{q_Q(Q'|Q)}{q_Q(Q|Q')} = \frac{p_Q(Q'|D)}{p_Q(Q|D)}$$
with respect to the posterior over Q values defined in Algorithm 1. Then the associated implicit Markov chain over rewards also satisfies the detailed balance condition with respect to the posterior $p_R(R|D)$.
\end{theorem}

\begin{proof}
Let $q_Q$ be the transition kernel over Q-values that satisfies the detailed balance condition with respect to the posterior $p_Q(Q|D)$ as assumed in the theorem statement.

The implicit transition kernel $q_R$ over rewards induced by $q_Q$ can be expressed as
\begin{equation}
q_R(R'|R) = q_Q(Q(R')|Q(R)) \left|\det\left(\frac{\partial Q(R')}{\partial R'}\right)\right|
\end{equation}
where $Q(R)=(I-\gamma \bar{P})^{-1}R$ is the Q-value corresponding to reward $R$ as used in Algorithm 1. The determinant term accounts for the change of variables from $Q$ to $R$.

The posterior over rewards can be expressed in terms of the posterior over Q-values as
\begin{equation}
p_R(R|\D) = p_Q(Q(R)|\D) \left|\det\left(\frac{\partial Q(R)}{\partial R}\right)\right| \
= p_Q(Q(R)|\D) \left|\det(I-\gamma \bar{P})^{-1}\right|.
\end{equation}

Now consider the ratio of the implicit transition kernel:
\begin{multline}
\frac{q_R(R'|R)}{q_R(R|R')} = \frac{q_Q(Q(R')|Q(R))}{q_Q(Q(R)|Q(R'))}  \frac{\left|\det\left(\frac{\partial Q(R')}{\partial R'}\right)\right|}{\left|\det\left(\frac{\partial Q(R)}{\partial R}\right)\right|} \
= \frac{p_Q(Q(R')|D)}{p_Q(Q(R)|\D)}  \frac{\left|\det\left(\frac{\partial Q(R')}{\partial R'}\right)\right|}{\left|\det\left(\frac{\partial Q(R)}{\partial R}\right)\right|} \
= \\
\frac{p_R(R'|\D) \det((I-\gamma\bar P')^{-1})}{p_R(R|\D) \det((I-\gamma\bar P)^{-1})}\frac{\det(I-\gamma\bar P')}{ \det(I-\gamma\bar P)}
= \frac{p_R(R'|D)}{p_R(R|D)}
\end{multline}
where the second equality follows from the assumed detailed balance condition on $q_Q$, the last equality follows from the expression for $p_R(R|D)$ derived above, and $\bar P'$ are the joint state-action transitions corresponding to $Q'$.
Thus, the implicit Markov chain over rewards induced by the transition kernel $q_Q$ satisfies detailed balance with respect to the posterior $p_R(R|D)$, as claimed.
\end{proof}

The theorem establishes an important property of the ValueWalk method, namely that the implicit Markov chain over rewards induced by the HMC-based sampling of Q-values satisfies detailed balance with respect to the true posterior over rewards given the demonstrations, $p_R(R|D)$. This property is crucial for the soundness of the method.

Detailed balance is a sufficient condition for the Markov chain to have a stationary distribution equal to the target distribution, in this case $p_R(R|D)$. This means that, assuming the chain is ergodic, the samples of rewards obtained from the ValueWalk method will asymptotically follow the true posterior distribution, regardless of the initial distribution. In other words, the theorem guarantees that, given enough samples, ValueWalk will correctly characterize the posterior uncertainty over rewards, which is a key goal of Bayesian inverse reinforcement learning.

\section{Variations of the baseline methods}

\subsection{PolicyWalk-HMC}
\label{app:pw-hmc}
We proposed that ValueWalk be used with Hamiltonian Monte Carlo (HMC) treating the underlying parameters as fully continuous. By contrast, PolicyWalk, as originally proposed, samples the next proposed value of the reward parameters from neighbours of the current point on a discretized grid. To isolate the speed-up effect of our Q-space trick from the speed-up due to HMC, we also implemented a version of PolicyWalk with HMC (denoted by PolicyWalk-HMC in the paper. This involves calculating the gradient of the posterior with respect to the reward parameters. To do that, we use the matrix-multiplication computed Q-values. We omit the dependence of the combined transition-policy matrix to the gradient, since the derivative of the optimal policy with respect to the reward is zero almost everywhere.

\subsection{Model-based AVRIL}
\label{app:mb-avril}
In the gridworld experiments, both PolicyWalk and ValueWalk are leveraging the environment dynamics, which AVRIL does not use. For fairer comparison, we are thus including also a model-based version of AVRIL, which differs from the original (model-free) AVRIL in that it evaluates the KL divergence from the prior across all states (the gridworld is using state-only rewards), and the TD term is calculated (1) over state action pairs and (2) the next-state value can be estimated using the actual expectation, instead of just using the Q-value of the next empirical state-action pair from the demonstrations. The remainder of the algorithm remains the same.


\section{Experiment details}
For the gridworld experiments, we used a version of AVRIL learning a Q-value for each state-action pair and a mean and variance value for the reward in each state. For PolicyWalk, we ran inference over a reward vector containing a reward value per each state. For ValueWalk, we ran inference over the state-value vector.

In the continuous state space environments, for the 3 continuous baseline methods, we match the setup from \cite{chan2021} and use neural network models with 2 hidden layers of 64 units and an ELU activation function. For our experiments, we scale up the network size with the complexity of the problem: we use one hidden layer with 8 units for 
Cartpole, 1 layer of 16 units for Acrobot, and 2 layers of 24 units for LunarLander. In each case, we also tried running AVRIL with a matching network size but in each case it performed similarly or usually worse than the default 2x64 setup for which results are reported.

For PolicyWalk and ValueWalk, we use the Pyro \citep{bingham2018} implementation of HMC+NUTS. For the control environment experiments, we ran with 2,000 warm-up steps and 10,000 inference steps for Lunar Lander and 2,000 warm-up steps with 5,000 inference steps for Carpole and Acrobot (since we are inferring fewer network parameters there). We automatically tune the step size during warm-up but do not tune the mass matrix. 

In the continuous environments, we use a Gaussian process prior with an RBF kernel with fixed scale of 1 and fixed lengthscale of 0.2 for Cartpole and Acrobot and 0.03 for Lunar Lander (chosen manually based on the distribution of features in the demonstrations for each environment, where the lengthscale roughly corresponds to the std of one-step change in each feature). 

In Cartpole, Acrobot, and Lunar Lander, we reuse the demonstration sets provided by the authors of AVRIL. Each contains 1000 demonstration trajectories, from which we randomly chose a set of 100 test trajectories and then split the remaining examples into 5 training splits. We then re-ran each experiment for each number $n_{\text{traj}}=1,3,7,10,15$ of trajectories on the first $n_{\text{traj}}$ trajectories of each of the 5 splits and evaluated the resulting apprentice agent on 300 episodes of the environment. We report the mean and std across the splits and evaluations.

Unless otherwise stated, we use a Boltzmann rationality coefficient of 3.

\section{Additional details of results}

\subsection{Gridworld experiments}
\label{app:gw-details}

Figure~\ref{fig:3x3_2d_hists} shows 2-D histograms of pairwise joint posteriors over rewards of the 9 states of the gridworld. Two aspects of the expert's behaviour are captured by this plot and may not be obvious from the simple histograms in Figure~\ref{fig:3x3_gw}. Firstly, the agent heading to the terminal top right corner can be explained either by the reward there being positive, or by the reward in other states being negative, and thus the agent using the terminal state as a way to escape incurring further negative rewards. Secondly, note that practically all of the probability mass is placed on the reward of the obstacle tile being lower than that of the two tiles below, thus explaining the expert avoiding the obstacle tile.

The plot also clearly shows that the posterior is non-Gaussian (note especially the sharp edge expressing high confidence that the ratio of the two values does not cross a certain threshold) and thus could not be captured by the Gaussian-assuming variational prior. Also, the rewards of different states are, sometimes very tightly, correlated, so modelling them as independent would again be inappropriate.

\begin{figure}
    \centering
    \includegraphics[width=0.99\textwidth]{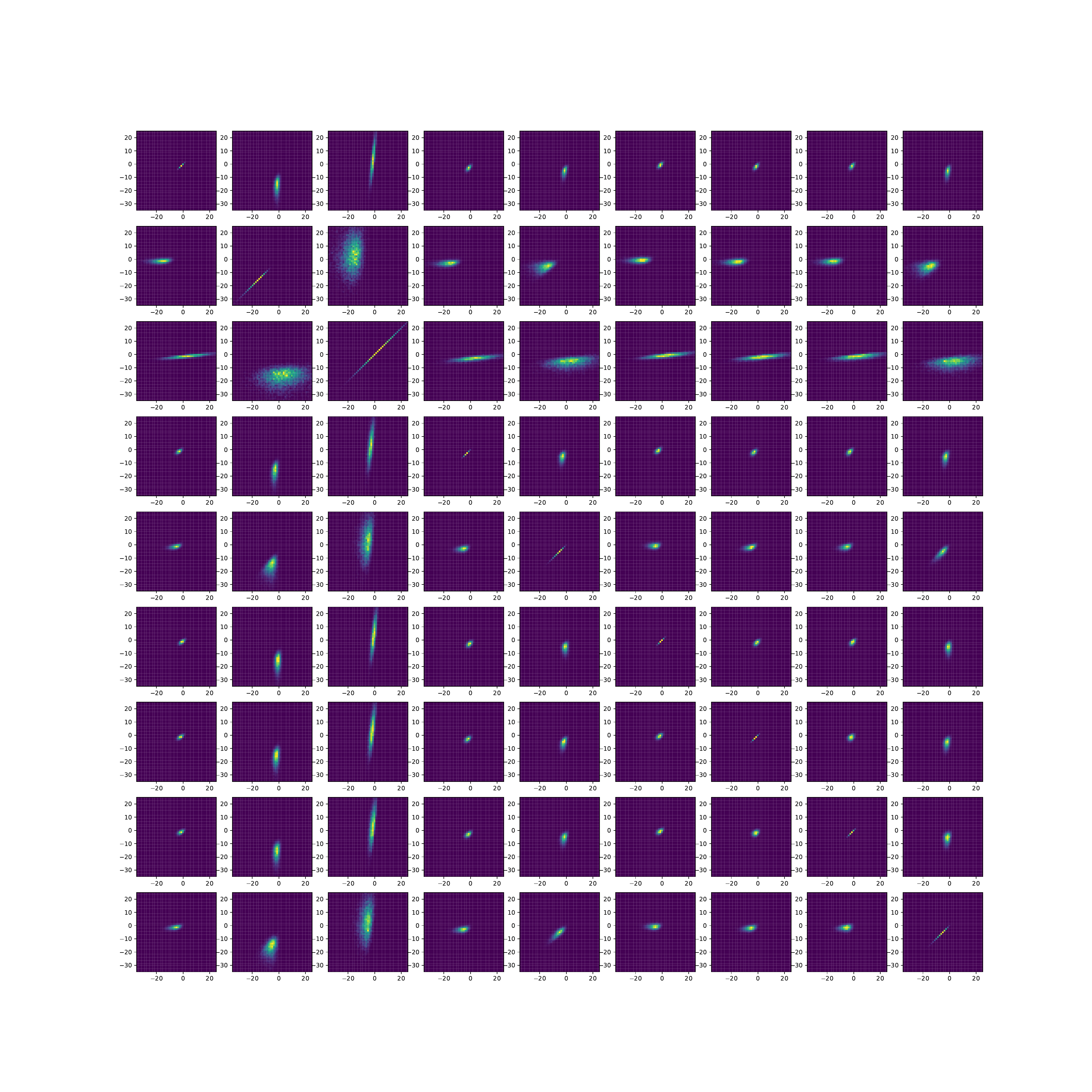}
    \caption{2-D histograms representing the joint posteriors of the rewards associated with the 9 states of the gridworld (enumerated left-to-right, top-to-bottom, so state 3 is the goal state in the top right corner.}
    \label{fig:3x3_2d_hists}
\end{figure}



\end{document}